\documentclass[a4paper]{llncs}

\usepackage{courier}

\usepackage{makeidx}  
\usepackage{float}
\usepackage{subfig}

\usepackage[british]{babel}
\usepackage[colorlinks=true,allcolors=black]{hyperref}
\usepackage[nameinlink]{cleveref}
\usepackage{listings}
\usepackage[backgroundcolor=lime]{todonotes}
\usepackage{enumitem}
\usepackage[linewidth=0.25pt]{mdframed}
\usepackage{xspace}
\usepackage{caption}

\let\oldtexttt\texttt
\renewcommand{\texttt}[1]{{\small{\oldtexttt{#1}}}}

\newcommand{\essence}[0]{\textsc{Essence}\xspace}
\newcommand{\conjure}[0]{\textsc{Conjure}\xspace}
\newcommand{\eprime}[0]{\textsc{Essence Prime}\xspace}
\newcommand{\savilerow}[0]{\textsc{Savile Row}\xspace}

\newcommand{\code}{\texttt}

\newcommand{\TITLE}{Towards Reformulating Essence Specifications for Robustness}
\title{\TITLE}
\newcommand{\AUTHOR}{Özgür Akgün, Alan M. Frisch, Ian P. Gent, Christopher Jefferson, Ian Miguel, Peter Nightingale, András Z. Salamon}
\hypersetup{pdftitle={\TITLE},pdfauthor={\AUTHOR},pdfsubject={\TITLE}}
\institute{
    University of St Andrews. \\ \{ozgur.akgun, ian.gent, caj21, ijm, Andras.Salamon\}@st-andrews.ac.uk \and
    University of York. \{alan.frisch, peter.nightingale\}@york.ac.uk
}

\author{
    \"Ozg\"ur Akg\"un\inst{1} \and
    Alan M. Frisch\inst{2} \and
    Ian P. Gent\inst{1} \and
    Christopher Jefferson\inst{1} \and \\
    Ian Miguel\inst{1} \and
    Peter Nightingale\inst{2} \and
    Andr\'as Z. Salamon\inst{1}
}
\institute{
    University of St Andrews. \\ \{ozgur.akgun, ian.gent, caj21, ijm, Andras.Salamon\}@st-andrews.ac.uk \and
    University of York. \{alan.frisch, peter.nightingale\}@york.ac.uk
}

\begin{document}

\maketitle

\begin{abstract}

The  \essence{} language  allows a user to specify a constraint problem at a level of abstraction above that at which constraint modelling decisions are made.
\essence{} specifications are refined into constraint models using the \conjure{} automated modelling tool, which employs a suite of refinement rules. 
However, \essence{} is a rich language in which there are many equivalent ways to specify a given problem.
A user may therefore omit the use of domain attributes or abstract types, resulting in fewer refinement rules being applicable and therefore a reduced set of output models from which to select.
This paper addresses the problem of recovering this information automatically to increase the robustness of the quality of the output constraint models in the face of variation in the input \essence{} specification.
We present reformulation rules that can change the type of a decision variable or add attributes that shrink its domain.
We demonstrate the efficacy of this approach in terms of the quantity and quality of models \conjure{} can produce from the transformed specification compared with the original.

\end{abstract}

\section{Introduction and Background}

The {\em modelling bottleneck} is the difficulty of formulating a problem of interest as a constraint model suitable for input to a constraint solver. 
The space of possible models for a given problem is typically large, and the model selected can have a dramatic effect on the efficiency of constraint solving. 
This presents a serious challenge for the inexpert user, who has difficulty in formulating a good (or even correct) model, and motivates efforts to automate constraint modelling.
In this paper we show that one source of difficulty to inexpert users can be ameliorated by \textit{type strengthening}  rules.

In this paper our focus is on the refinement-based approach, where a user writes {\em abstract} constraint specifications that describe a problem above the level at which constraint modelling decisions are made. Abstract constraint specification languages, such as \essence{} or Zinc, support abstract decision variables with types such as set, multiset, relation and function, as well as {\em nested} types, such as set of sets and multiset of relations. Problems can typically be specified very concisely in this way, as demonstrated by the examples in \Cref{essence-rel}. However, existing constraint solvers do not support these abstract decision variables directly, so abstract constraint specifications must be {\em refined} into concrete constraint models.

Work on automation of aspects of constraint modelling
can be grouped into distinct categories.
Models can be learned from positive or negative examples~\cite{de2018learning,Bessiere2017constraint,arcangioli2016multiple},
various kinds of queries~\cite{beldiceanu_model_2012mod,bessiere2013constraint,bessiere2014boosting},
arguments~\cite{argumentationConstraintAcquistion},
or from natural language descriptions~\cite{kiziltan2016constraint}.
It is possible to partially automate the  transformation of medium-level solver-independent constraint models~\cite{rendl2010effective,nethercote2007minizinc,vanhentenryck1999opl,mills1999eacl,nightingale2014automatically,nightingale2017automatically,nightingale2015automatically,little2003using}.
Closer to our work,
implied constraints have been derived from a constraint model~\cite{frisch2003cgrass,colton2001constraint,charnley2006automatic,bessiere2007learningmod,leo2013globalizingmod} and
refinement of abstract constraint specifications has been considered~\cite{frisch2005essence} using the languages ESRA~\cite{flener2003esra}, \essence{}~\cite{frisch2013:essence}, ${\mathcal F}$~\cite{hnich2003thesis} and Zinc \cite{marriott2008design,koninck2010data,rafeh2016linzinc}. 

We here use \essence{} \cite{frisch2013:essence} as our abstract problem specification language.
\essence supports abstract decision variables with \emph{types} such as \code{set}, \code{mset} (denoting a multiset), \code{relation} and \code{function}, as well as \emph{nested} types, such as \code{set of set} and \code{mset of relation}.
Types are used to construct \emph{domains}, which are abstract collections of objects of some common type but with possibly additional structure indicated by means of \emph{domain attributes}.
Problems can typically be specified very concisely in this way.
Abstract constraint specifications must be {\em refined} into concrete constraint models for existing constraint solvers. Our \conjure{} system \cite{akgun2011extensible,akgun2013automated,akgun2014breaking} employs refinement rules to convert  an \essence{} specification into the solver-independent constraint modelling language \eprime{} \cite{rendl2010effective}. From \eprime{} we use \savilerow{} \cite{nightingale2017automatically}
to translate the model into input for a particular constraint solver while performing solver-specific model optimisations.

\essence{} is a rich language in which there are many equivalent ways to specify a given problem. It is possible, therefore, for a user to avoid the use of domain attributes or abstract types, resulting in fewer refinement rules being applicable and therefore a reduced set of output models from which to select. This paper addresses the problem of recovering this information automatically and hence increasing the robustness of the quality of the output constraint models.
We present reformulation rules that can change the type of a decision variable or add attributes that shrink its domain. We demonstrate the efficacy of this approach in terms of the quantity and quality of models \conjure{} can produce from the transformed specification compared with the original.
All the methods described in this paper are implemented as part of the \conjure{} system.\footnote{An archive containing several examples of robustness transformations performed by \conjure can be found at the following repository: \\ \url{https://github.com/stacs-cp/ModRef2021-robustness}}

\section{Motivating Examples}

\essence is a rich language with a wide range of type constructors and domain attributes. 
\conjure uses type constructors and domain attributes when selecting from its library of representations. 
\conjure has special highly efficient representations for many specific domains. For example, sets of fixed size and total functions. These are dramatically more efficient during constraint solving than the representations for variable-size sets and partial functions. Hence it is absolutely vital that the type constructor and domain attributes are as specific as possible. In this section we will give some concrete examples of how this can be achieved. 

In \essence, a \emph{domain attribute} (denoted in brackets after the name of a type in a domain) specialises the domain. 
For example, a function can be \texttt{surjective}, \texttt{injective} or \texttt{total} (\essence functions are partial by default). The specification given in Figure \ref{fig:funtypesA} defines a partial surjective function. A partial surjective function from a set \(S\) to a set \(T\) where \(|S| = |T|\) must necessarily be a total bijective function. Therefore we can strengthen this function to be \texttt{function (bijective, total) Index --> Index}. This has not changed the number of values in the domain, but \conjure has specialised efficient representations for \texttt{total} functions, so this change can lead to dramatically better performance during constraint solving. 


\newsavebox{\funtypesA}
\begin{lrbox}{\funtypesA}
\begin{minipage}{\textwidth}
\begin{lstlisting}
given n : int(1..)
letting Index be domain int(1..n)
find arrangement : function (surjective) Index --> Index
\end{lstlisting}
\end{minipage}
\end{lrbox}

\newsavebox{\funtypesB}
\begin{lrbox}{\funtypesB}
\begin{minipage}{\textwidth}
\begin{lstlisting}
given n : int(1..)
letting Index be domain int(1..n)
find arrangement : function (total, bijective) Index --> Index
\end{lstlisting}
\end{minipage}
\end{lrbox}


\newsavebox{\relA}
\begin{lrbox}{\relA}
\begin{minipage}{\textwidth}
\begin{lstlisting}
find x : relation of (int(1..3) * int(4..6) * int(7..9))
such that forAll i : int(1..3) . forAll k : int(7..9) . x(i,_,k) = {$\epsilon$}
\end{lstlisting}
\end{minipage}
\end{lrbox}

\newsavebox{\relB}
\begin{lrbox}{\relB}
\begin{minipage}{\textwidth}
\begin{lstlisting}
find x: function (total) (int(1..3), int(7..9)) --> int(4..6)
such that forAll i : int(1..3) . forAll k : int(7..9) . toRelation(x)(i,k,_) = {$\epsilon$}
\end{lstlisting}
\end{minipage}
\end{lrbox}

\newsavebox{\relC}
\begin{lrbox}{\relC}
\begin{minipage}{\textwidth}
\begin{lstlisting}
find x: function (total) (int(1..3), int(7..9)) --> int(4..6)
such that forAll i : int(1..3) . forAll k : int(7..9) . x(i,k) = $\epsilon$
\end{lstlisting}
\end{minipage}
\end{lrbox}

\begin{figure}
\subfloat[An \essence specification using \texttt{function} domains, before domain strengthening\label{fig:funtypesA}]{ \usebox{\funtypesA} } \\
\subfloat[An \essence specification using \texttt{function} domains, after domain strengthening\label{fig:funtypesB}]{ \usebox{\funtypesB} }

\subfloat[Input \essence{} specification with a \texttt{relation} decision variable\label{fig:relA}]{ \usebox{\relA} } \\
\subfloat[Reformulating the above to a \texttt{function} domain and the \texttt{toRelation} operator\label{fig:relB}]{ \usebox{\relB} } \\
\subfloat[Reformulating the above to remove the \texttt{toRelation} operator\label{fig:relC}]{ \usebox{\relC} }

\caption{\essence{} domain and representation rule examples}
\label{essence-rel}\label{fig:funrules}
\end{figure}

The type of a domain can also make a dramatic difference to the quality of models produced by \conjure. Our second class of reformulations identify when a type may be replaced by a more specific type (for example, replacing a multiset with a set). 

The \essence specification in Figure~\ref{fig:relA} contains a relation variable that would be better posed as a function. For each assignment to the first and third value in the relation, the relation is satisfied by exactly one assignment to the middle index. In this example, $\epsilon$ can be an arbitrary \essence{} expression possibly involving other decision variables. The second index is functionally defined by the first and third indices, so we can replace the relation with a function as shown in Figure~\ref{fig:relB}.  To simplify the process of changing the relation to a function, we have replaced all occurrences of the relation \texttt{x} in the constraints with \texttt{toRelation(x)}. The \texttt{toRelation} operator maps the function back into a relation, and allows us to be sure we can transform any constraint involving the relation \texttt{x} into a constraint on the function \texttt{x}. We also reordered the indices of the relation to place the functionally defined index (or indices) at the end. 

Figure~\ref{fig:relB} has reduced the domain size of the variable in our specification from \(2^{3^3}=134,217,728\) to \(3^9 = 19,683\). Reducing the size of a domain strongly indicates (but does not guarantee) that the specification has been improved, because we have a much smaller domain to represent and search in our constraint solver. The final step in the reformulation is removing the \texttt{toRelation} operators where possible. This results in the final specification given in Figure~\ref{fig:relC}.

\section{Domain Strengthening via Attribute Recovery}

Refinement rules in \essence rewrite an abstract structure into a more concrete equivalent.
Domain attributes provide additional information that can be used for more effective refinement.
This information can also be provided in an \essence{} specification in the form of constraints.
However, \conjure uses only the domain and attributes of a variable to select its representation, so using constraints instead of domain attributes reduces the quantity and quality of models output by \conjure.
In this section we describe how an \essence{} specification can be reformulated to recover latent attribute information that was omitted from the specification.
There exists a strengthening rule to recover each attribute in \essence{}, as we will demonstrate below.

A \emph{strengthening rule} takes as input a domain and optionally a constraint expression and outputs a new domain.
In addition, it indicates whether the input constraint should remain (by default it is removed).
For some rules the entire constraint is subsumed by the addition of the attribute so the constraint is not required, in others the constraint is still required.
The strengthening rules are applied by \conjure{} before the representation of each variable is chosen.
The strengthening rules presented in this paper take a similar form to the \conjure rules presented previously \cite{akgun2011extensible}. Each rule takes as input a single decision variable and one or more constraints, and outputs a new modified variable and one or more constraints. Each rule removes the constraints matched by the input of the rule unless we state otherwise.  All other constraints involving the variable are unchanged, except referring to the new name for the variable.

\conjure rules contain meta-variables to denote expressions which must be matched, such as the \texttt{\&n} in \Cref{set-size-rule}. These meta-variables can match an arbitrary \essence expression, not just a single identifier. This means our rules can match large \essence expressions with simple patterns. Because the identifier \texttt{\&n} in \Cref{set-size-rule} is also used in the \texttt{size} attribute of a \texttt{set}, it must be a constant or parameter, not a decision variable. The rule matcher will automatically reject any \texttt{\&n} which does not meet this requirement. Other identifiers, such as \texttt{\&exp} in Figure \ref{function-total}, will match with any \essence expression at all, including expressions involving decision variables.

General \conjure rewrite rules (such as a rule which normalises \(A > B\) to \(B < A\)) can match partial expressions contained within constraints. Strengthening rules must only be applied to top-level constraints. 
In \essence abstract domain constructors can be nested arbitrarily, and attributes can be recovered for domains nested inside other domain constructors.
This is done by \conjure for all rules automatically, and does not require writers of rules to worry about deeply nested types.
The example given in \Cref{fig:DS_nested} shows an example where we place a constraint on all the \texttt{mset} members of S.
\conjure automatically recognises the \texttt{forAll m in S} quantifier imposes the constraint on all members of \texttt{S}, and so the rule in \Cref{mset-maxOccur-rule} which is designed for \texttt{mset} variables will trigger, adding to the members of \texttt{S} the attribute \texttt{(maxOccur 3)}, producing the output in \Cref{fig:DS_nested_out}.

\subsection{Recovering Size and Occurrence Related Attributes}\label{sec:recoversize}

In \essence{} abstract domain constructors --- \texttt{set}, \texttt{mset}, \texttt{function}, \texttt{relation}, and \texttt{partition} --- can have \texttt{size}, \texttt{minSize}, and \texttt{maxSize} attributes.
\conjure{} can recover these attributes when a cardinality constraint is posted at the top level in the problem specification.
For example, a constraint of the form \texttt{|x| = n} implies the recovery of a \texttt{size} attribute for \texttt{x};
the constraints \texttt{|x| < n} and \texttt{|x| <= n} imply a \texttt{maxSize} attribute with the values \texttt{n-1} and \texttt{n} respectively.
The recovery of the \texttt{minSize} attribute is handled similarly.

In addition to the three \emph{size} attributes common to all abstract domain constructors, \texttt{mset} and \texttt{partition} have additional attributes which are handled similarly.
For \texttt{mset}, the attributes \texttt{minOccur} and \texttt{maxOccur} constrain the number of occurrences of individual values.
For \texttt{partition} the attributes \texttt{partSize}, \texttt{minPartSize}, \texttt{maxPartSize} constrain the sizes of the parts in the partition, and \texttt{numParts}, \texttt{minNumParts} constrain the number.

\newsavebox{\setmsetin}
\begin{lrbox}{\setmsetin}
\begin{minipage}{\textwidth}
\begin{lstlisting}
find S : set of mset of int(0..9)
such that
   |S| = 2,
   forAll m in S . forAll i : int(0..9) . freq(m,i) <= 3
\end{lstlisting}
\end{minipage}
\end{lrbox}

\newsavebox{\setmsetout}
\begin{lrbox}{\setmsetout}
\begin{minipage}{\textwidth}
\begin{lstlisting}
find S: set (size 2) of mset (maxOccur 3) of int(0..9)
\end{lstlisting}
\end{minipage}
\end{lrbox}

\begin{figure}
\subfloat[The input\label{fig:DS_nested} \essence{} problem specification]{ \usebox{\setmsetin} } \\
\subfloat[Recovered attributes\label{fig:DS_nested_out}]{ \usebox{\setmsetout} }
\caption{An example \essence{} specification with nested domains.}
\label{set-mset-example}
\label{essence-mset}

\subfloat[Recovering the \texttt{size} attribute for sets\label{set-size-rule}]{
\begin{minipage}{\textwidth}
    \begin{mdframed}
    \begin{description}[leftmargin=!,labelwidth=\widthof{\textbf{AAAAA}}]
        \item[Input   ] \lstinline$find &x : set of &T$
        \item[Input   ] \lstinline$|&x| = &n$
        \item[Output  ] \lstinline$find &x : set (size &n) of &T$
    \end{description}
    \end{mdframed}
\end{minipage}
 } \\
\subfloat[Recovering the \texttt{maxOccur} attribute for multi-sets\label{mset-maxOccur-rule}]{
\begin{minipage}{\textwidth}
    \begin{mdframed}
    \begin{description}[leftmargin=!,labelwidth=\widthof{\textbf{AAAAA}}]
        \item[Input   ] \lstinline$find &x : mset of &T$
        \item[Input   ] \lstinline$forAll &i : &T . freq(&x,&i) <= &n$
        \item[Output  ] \lstinline$find &x : mset (maxOccur &n) of &T$
    \end{description}
    \end{mdframed}
\end{minipage}
}
\caption{Rules for the domains in Figure~\protect\ref{set-mset-example}.}
\label{set-mset-example-rules}
\end{figure}
\Cref{set-size-rule} presents a recovery of the \texttt{size} attribute for sets.
This rule as given here is specialised to just the \texttt{set} type constructor, however in \conjure{} this rule is implemented to handle any abstract domain constructor.
\Cref{mset-maxOccur-rule} presents a recovery of the \texttt{maxOccur} attribute for \texttt{mset}.
This rule, as well as all other strengthening rules also work when the \texttt{mset} domain is nested inside other abstract domain constructors. \Cref{fig:DS_nested} presents an example of a nested type where we recover attributes for both the inner and outer type constructors.
In \Cref{fig:DS_nested}, the domain given of \texttt{S} is infinite.
While \conjure requires all variables have a finite domain, the check for finiteness is performed after type-strengthening.
This shows another way in which type strengthening can help users model their problems more easily.

\subsection{Recovering Special \texttt{function} and \texttt{sequence} Attributes}

In addition to the common \texttt{minSize}, \texttt{maxSize} and \texttt{size} attributes,
functions have four additional attributes: \texttt{total}, \texttt{injective}, \texttt{surjective}, and \texttt{bijective}.
\Cref{function-total} gives a strengthening rule, where the \texttt{total} attribute is inferred if there is a constraint to assign values to all mappings in the function.
Here, unlike most other strengthening rules, the constraint is not removed because it contains more information than just representing a \texttt{total} attribute.
\begin{figure}
\subfloat[Recovering the \texttt{total} attribute for functions\label{function-total}]{
\begin{minipage}{\textwidth}
    \begin{mdframed}
    \begin{description}[leftmargin=!,labelwidth=\widthof{\textbf{AAAAA}}]
        \item[Input   ] \lstinline$find &x : function &T_1 --> &T_2$
        \item[Input   ] \lstinline$forAll &i : &T_1 . &x(&i) = &exp$
        \item[Output  ] \lstinline$find &x : function (total) &T_1 --> &T_2$\newline
            \emph{The constraint remains unchanged.}

    \end{description}
    \end{mdframed}
\end{minipage}
 } \\
\subfloat[Recovering the \texttt{surjective} attribute for functions\label{function-surjective}]{
\begin{minipage}{\textwidth}
    \begin{mdframed}
    \begin{description}[leftmargin=!,labelwidth=\widthof{\textbf{AAAAA}}]
        \item[Input   ] \lstinline$find &x : function &T_1 --> &T_2$
        \item[Input   ] \lstinline$forAll &j : &T_2 . exists &i : &T_1 . f(&i) = &j$
        \item[Output  ] \lstinline$find &x : function (surjective) &T_1 --> &T_2$
    \end{description}
    \end{mdframed}
\end{minipage}
 } \\
\subfloat[Recovering the \texttt{injective} attribute for functions\label{function-injective}]{
\begin{minipage}{\textwidth}
    \begin{mdframed}
    \begin{description}[leftmargin=!,labelwidth=\widthof{\textbf{AAAAA}}]
        \item[Input   ] \lstinline$find &x : function &T_1 --> &T_2$
        \item[Input   ] \lstinline$forAll &i,&j : &T_1 . &i != &j -> &x(&i) != &x(&j)$
        \item[Output  ] \lstinline$find &x : function (injective) &T_1 --> &T_2$
    \end{description}
    \end{mdframed}
\end{minipage}
}
\caption{Strengthening rules for attributes of function domains}
\end{figure}
\Cref{function-surjective} gives a strengthening rule, where the \texttt{surjective} attribute is inferred if there is a constraint stating for all values in the range of the function there is a mapping.
Similarly, \Cref{function-injective} gives a strengthening rule where the \texttt{injective} attribute is inferred if there is a constraint stating the image of the function to be distinct for distinct values.
In both of these rules, the constraints are removed from the model because they are subsumed by adding the suitable attribute to \texttt{x}. \conjure will further infer the \texttt{bijective} attribute for any variable with both the \texttt{injective} and \texttt{surjective} attributes.

Sequence domains support the \texttt{injective}, \texttt{surjective} and \texttt{bijective} attributes as well. These are handled similarly to functions.

\subsection{Recovering Special \texttt{relation} Attributes}\label{ref:relationDS}

\Cref{relation-functional} gives strengthening rules to infer \texttt{functional} and \texttt{total\_functional} attributes.
These two attributes restrict the domain of a relation so some columns of a relation are functionally determined by the rest of the columns.
\texttt{functional} restricts the functionally determined columns take at most one assignment for each assignment to the other columns, \texttt{total\_fuctional} restricts the functionally defined columns to take exactly one assignment.
For example, a binary relation \lstinline$r$ together with the constraint \lstinline$forAll i : dom . |r(i,_)| = 1$ can be turned into a function mapping values from the first column to the second one.
In addition, such a function domain has to be \texttt{total}, because there is exactly one value of the second column for each value of the first.
The constraint 
\lstinline$forAll i : dom . |r(i,_)| <= 1$ would let \conjure{} recover a \texttt{functional} attribute instead of \texttt{total\_functional}.

\begin{figure}
\subfloat{
\begin{minipage}{\textwidth}
    \begin{mdframed}
    \begin{description}[leftmargin=!,labelwidth=\widthof{\textbf{AAAAA}}]
        \item[Input   ] \lstinline$find &x : relation of (&T_1 * &T_2 * &T_3)$
        \item[Input   ] \lstinline$|&x(&a,&b,_)| <= 1$
        \item[Output  ] \lstinline$find &x : relation (functional (1,2)) of (&T_1 * &T_2 * &T_3)$
    \end{description}
    \end{mdframed}
\end{minipage}
 } \\
\subfloat{
\begin{minipage}{\textwidth}
    \begin{mdframed}
    \begin{description}[leftmargin=!,labelwidth=\widthof{\textbf{AAAAA}}]
        \item[Input   ] \lstinline$find &x : relation of (&T_1 * &T_2 * &T_3)$
        \item[Input   ] \lstinline$|x(&a,&b,_)| = 1$
        \item[Output  ] \lstinline$find &x : relation (total_functional (1,2)) of (&T_1 * &T_2 * &T_3)$
    \end{description}
    \end{mdframed}
\end{minipage}
 } \\
\subfloat{
\begin{minipage}{\textwidth}
    \begin{mdframed}
    \begin{description}[leftmargin=!,labelwidth=\widthof{\textbf{AAAAA}}]
        \item[Input   ] \lstinline$find &x : relation of (&T_1 * &T_2 * &T_3)$
        \item[Input   ] \lstinline$&x(&a,&b,_) = {&c}$
        \item[Output  ] \lstinline$find &x : relation (total_functional (1,2)) of (&T_1 * &T_2 * &T_3)$\newline
                    \emph{The constraint remains unchanged.}

    \end{description}
    \end{mdframed}
\end{minipage}
}
\caption{Strengthening rules for attributes of relation domains\label{relation-functional}}
\end{figure}
\begin{figure}
\begin{minipage}{\textwidth}
    \begin{mdframed}
    \begin{description}[leftmargin=!,labelwidth=\widthof{\textbf{AAAAA}}]
        \item[Input   ] \lstinline$find &x : partition from &T$
        \item[Input   ] \lstinline$forAll &i,&j in parts(&x) . |&i| = |&j|$
        \item[Output  ] \lstinline$find &x : partition (regular) from &T$
    \end{description}
    \end{mdframed}
\end{minipage}
\caption{Strengthening rules for attributes of partition domains\label{partition-rules}}
\end{figure}

\subsection{Recovering Special \texttt{partition} Attributes}

\Cref{partition-rules} gives strengthening rules for the \texttt{partition} type constructor. In the first the \texttt{regular} attribute is inferred if there is a constraint forcing all parts in the partition to have the same cardinality.
As discussed earlier in \Cref{sec:recoversize}, \conjure also recovers the \texttt{numParts} and \texttt{partSize} attributes for \texttt{partition}.

\subsection{Domain-Only Recovery}

It is sometimes possible to recover domain attributes just from the domain of a variable.
For example in \Cref{fig:funtypesA}, a \texttt{surjective} function between two domains of equal size must be \texttt{total} and \texttt{bijective}.
\conjure contains a range of similar rules for other types. For example for a \texttt{mset} domain \texttt{maxSize n} implies \texttt{maxOccur n} and \texttt{minOccur n} implies \texttt{minSize n}.
Futher, for partitions if the \texttt{partSize} multiplied by the \texttt{numParts} is equal to the size of the set the partition is defined over, the partition is \texttt{complete} and \texttt{regular}.
While these modifications do not remove any constraints, they allow \conjure to choose from a wider range of representations, as representation selection is based upon the domain and attributes.

\subsection{Limitations of Attribute Recovery}

Each of \conjure's attribute recovery rules match an explicit pattern of constraints and attributes.
The reformulation rules in \conjure can reduce other constraints to fit into these patterns, but there are occasions where we fail to detect attributes.
Theorem \ref{thm:halting} shows this is inevitable, as detecting if we can add an attribute to a variable is equivalent to the halting problem.

\begin{theorem}\label{thm:halting}
For any \essence attribute $A$, suppose there is an oracle that can decide, for any \essence specification $S$ containing variable \(V\), whether \(V\) can have the attribute \(A\) attached to it without removing solutions. 
Then this oracle can be used to solve the halting problem, unless \(A\) is satisfied by every variable it can be attached to.
\end{theorem}
\begin{proof}
To reduce from the Halting problem, consider determining if a Turing machine \(T\), with a distinguished state \(q\) and no halting states, ever reaches \(q\) when started on a blank tape.
This problem is \({\Sigma}^1_0\)-complete.
We construct an \essence specification which takes a parameter \(n\) and has a solution if \(T\) reaches \(q\) within \(n\) steps, since \(T\) cannot reach any part of the tape that is more than \(n\) steps from the starting position.
Detecting if this \essence specification has a solution for any \(n\) is therefore equivalent to solving the halting problem.
Assuming that $A$ is not satisfied by some variable $V$, we now add $V$ to this \essence specification, placing no constraints on the variable \(V\). 
If \(T\) never reaches state \(q\), the problem has no solutions and so we can add any extra attribute to \(V\) without affecting the set of solutions.
If \(T\) does reach state \(q\) after some number of steps \(n\), then the set of solutions includes (possibly many copies of) all assignments to \(V\), so the addition of $A$ will reduce the set of solutions and is invalid.
The supposed oracle therefore solves the halting problem.
\end{proof}

\section{Type Strengthening}

As well as adding attributes to domains, \conjure{} also implements \emph{type strengthening} rules which transform one domain into another, e.g. turning a \texttt{relation} into a \texttt{function} or \texttt{mset} into \texttt{set}.
This provides a greater set of representational choices.

\subsection{Overview}

Our implementation of type strengthening builds directly upon domain strengthening.
We split type strengthening into two parts.
Firstly, for each type \texttt{T} which can be strengthened to another type \texttt{U}, there is an attribute on \texttt{T} which restricts a variable of type \texttt{T} to assignments of type \texttt{U}, and every other attribute on \texttt{T} is also valid on \texttt{U}.
This means detecting type strengthening uses only the domain of a variable. Further, for each type \texttt{T} which can be strengthened into another type \texttt{U}, we have an operator \texttt{toT} which transforms a \texttt{U} back into a \texttt{T}.
This allows us to replace a variable \texttt{T t} with a variable \texttt{U u}, replacing all occurrences of \texttt{t} with \texttt{toT(u)}.
We then use \conjure's standard rewrite engine to simplify the resulting specification.

\subsection{MSet to Set}

The rule below demonstrates transforming a \texttt{mset} with \texttt{(maxOccur 1)} into a \texttt{set}.
When applying this rule every occurrence of \texttt{\&x} in the model is replaced with \texttt{toMSet(\&x)}.
For example, \texttt{freq(\&x,2) = 0} turns into \texttt{freq(toMSet(\&x),2) = 0}.
Later \conjure simplifies this expression to \texttt{!(2 in \&x)} by applying one of its rewrite rules.

\begin{mdframed}
\begin{description}[leftmargin=!,labelwidth=\widthof{\textbf{AAAAA}}]
    \item[Input   ] \lstinline$find &x : mset (maxOccur 1) of &T$
    \item[Output  ] \lstinline$find &x : set of &T$
\end{description}
\end{mdframed}

\subsection{Relation to Function}

We previously discussed in \Cref{ref:relationDS} how we can detect relations which are functional in one or more of their indices.
Such relations can be transformed into functions.
In general these transformations involve arbitrary arity relations with an arbitrary subset of the indices of the relation defining the domain of the function. In this section we will provide some concrete examples for arity 3 relations.

The two rules below show examples of transforming a \texttt{relation} to a partial and total \texttt{function}.
When transforming a \texttt{relation} into a \texttt{function}, we replace all occurrences of the relation \texttt{x} with the expression \texttt{toRelation(x)}.
This leads to specifications like \Cref{fig:relB}, where we use \texttt{toRelation} and project the resulting relation. \conjure simplifies this example to the specification given in \Cref{fig:relC}.

    \begin{mdframed}
    \begin{description}[leftmargin=!,labelwidth=\widthof{\textbf{AAAAA}}]
        \item[Input   ] \lstinline$find &x : relation (functional (1,2)) of (&T_1 * &T_2 * &T_3)$
        \item[Output  ] \lstinline$find &x : function (&T_1 * &T_2) --> &T_3$
    \end{description}
    \end{mdframed}
    \begin{mdframed}
    \begin{description}[leftmargin=!,labelwidth=\widthof{\textbf{AAAAA}}]
        \item[Input   ] \lstinline$find &x : relation (total_functional (1,3)) of (&T_1 * &T_2 * &T_3)$
        \item[Output  ] \lstinline$find &x : function (total) (&T_1 * &T_3) --> &T_2$
    \end{description}
    \end{mdframed}

\section{Conclusions}

We have shown how we can make the powerful and expressive type system of \essence more robust and easier to use. Type and domain strengthening make \essence more useful for beginners who have not yet learnt, or do not wish to learn, the full list of attributes available in \essence. Furthermore, automated type and domain strengthening allows for new attributes to be added to the \essence language which can improve the performance of an \essence specification without the user having to make any changes to their model.

In future work, we plan to extend type and domain strengthening to more general kinds of rewriting of \essence specifications.
For instance, currently we cannot recover the injectivity of \code{f} and \code{g} from

\noindent
\begin{minipage}{\textwidth}
\begin{lstlisting}
forAll i,j . f[i] = j <-> g[j] = i
\end{lstlisting}
\end{minipage}
\newline
via our existing type and domain strengthening rules.

\subsubsection*{Acknowledgements}
This research was supported by the UK EPSRC grants EP/K015745/1 and EP/V027182/1. Chris Jefferson is a University Research Fellow funded by the Royal Society.

\bibliography{new}
\bibliographystyle{splncs04}

\end{document}